\documentclass[11pt]{article}

\newif\ifdraft
\drafttrue
\usepackage{chngcntr}
\counterwithout{figure}{section}
\usepackage[margin=1.15in]{geometry}

\usepackage{bbm}
\usepackage[utf8]{inputenc} % allow utf-8 input
\usepackage[T1]{fontenc}    % use 8-bit T1 fonts
\usepackage{hyperref}       % hyperlinks
\usepackage{url}            % simple URL typesetting
\usepackage{amsfonts}       % blackboard math symbols
\usepackage{nicefrac}       % compact symbols for 1/2, etc.
\usepackage{setspace}
\usepackage{xcolor}
\usepackage{xspace}

\usepackage{graphicx}
\usepackage{empheq}
\usepackage{amsfonts}
\usepackage{amsthm}
\usepackage{algorithm}
\usepackage{bbm}
\usepackage[noend]{algpseudocode}
\usepackage{caption}
\usepackage{subcaption}
\newtheorem{theorem}{Theorem}
\newtheorem{definition}[theorem]{Definition}
\newtheorem{lemma}[theorem]{Lemma}
\newtheorem{corollary}[theorem]{Corollary}

\usepackage{algorithm}
\floatname{algorithm}{Optimization Problem}
\renewcommand{\S}{\mathbb{S}}
\newcommand{\D}{{\cal D}}

\title{Eigenvalue Decay Implies Polynomial-Time Learnability for Neural Networks}

\author{Surbhi Goel \thanks{Department of Computer Science, UT-Austin,
    \texttt{surbhi@cs.utexas.edu}. Work supported by a Microsoft Data
    Science Initiative Award.}\and Adam Klivans\thanks{Department of Computer Science,
    UT-Austin,    \texttt{klivans@cs.utexas.edu}. Part of this work was done while
    visiting the Simons Institute for Theoretical Computer Science.}}

\date{\today}

\begin{document}

\maketitle
\thispagestyle{empty}

\begin{abstract}

  We consider the problem of learning function classes computed by
  neural networks with various activations (e.g. ReLU or Sigmoid), a
  task believed to be computationally intractable in the worst-case.
  A major open problem is to understand the minimal assumptions under
  which these classes admit provably efficient algorithms. In this work we show
  that a natural distributional assumption corresponding to {\em
    eigenvalue decay} of the Gram matrix yields polynomial-time
  algorithms in the non-realizable setting for expressive classes of
  networks (e.g. feed-forward networks of ReLUs).  We make no 
  assumptions on the structure of the network or the labels.  Given
  sufficiently-strong polynomial eigenvalue decay, we obtain {\em
    fully}-polynomial time algorithms in {\em all} the relevant
  parameters with respect to square-loss. Milder decay assumptions
  also lead to improved algorithms.  This is the first purely
  distributional assumption that leads to polynomial-time algorithms
  for networks of ReLUs, even with one hidden layer.  Further, unlike
  prior distributional assumptions (e.g., the marginal distribution is
  Gaussian), eigenvalue decay has been observed in practice on common
  data sets.

%  Our algorithm, a combination of recent work on column sampling with
 % Kernel methods, applies to any function class that can be embedded
  %in a suitable Reproducing Kernel Hilbert Space (RKHS).  The main
  %technical contribution is a new approach for proving generalization
  %bounds for kernelized regression using {\em Compression Schemes}
  %as opposed to Rademacher bounds.  In
  %general, it is known that sample-complexity bounds for kernel
 % methods must depend on the norm of the corresponding RKHS, which can
 % quickly become large depending on the kernel function employed.  We
 % sidestep these worst-case bounds by {\em sparsifying} the Gram
 % matrix using recent work on recursive Nystr\"{o}m sampling due to Musco
 % and Musco~\cite{musco2016recursive}.  We prove that our approximate, sparse hypothesis admits a compression scheme whose true error depends on the rate of
 % eigenvalue decay.
\end{abstract}
\setcounter{page}{0}
\newpage

\section{Introduction}
%!TEX root = ./nips_2017.tex

Understanding the computational complexity of learning neural networks from random examples is a fundamental problem in machine learning. Several researchers have proved results showing computational {\em hardness} for the worst-case complexity of learning various networks-- that is, when no assumptions are made on the underlying distribution or the structure of the network \cite{Daniely16, goel2016reliably, KS09, LivniSS14, zhang2016l1}.  As such, it seems necessary to take some assumptions in order to develop efficient algorithms for learning deep networks (the most expressive class of networks known to be learnable in polynomial-time without any assumptions is a sum of one hidden layer of sigmoids \cite{goel2016reliably}).  A major open question is to understand what are the ``correct'' or minimal assumptions to take in order to guarantee efficient learnability\footnote{For example, a very recent paper of Song, Vempala, Xie, and Williams \cite{sxvw} asks ``What form would such an explanation take, in the face of existing complexity-theoretic lower bounds?''}.  An oft-taken assumption is that the marginal distribution is equal to some smooth distribution such as a multivariate Gaussian.  Even under such a distributional assumption, however, there is evidence that fully polynomial-time algorithms are still hard to obtain for simple classes of networks \cite{KlivansK14,sxvw}.  As such, several authors have made further assumptions on the underlying structure of the model (and/or work in the noiseless or {\em realizable} setting). 

In fact, in an interesting recent work, Shamir \cite{shamir2016distribution} has given evidence that both distributional assumptions and assumptions on the network structure are necessary for efficient learnability using gradient-based methods.  Our main result is that under {\em only} an assumption on the marginal distribution, namely eigenvalue decay of the Gram matrix, there exist efficient algorithms for learning broad classes of neural networks even in the non-realizable (agnostic) setting with respect to square loss.  Furthermore, eigenvalue decay has been observed often in real-world data sets, unlike distributional assumptions that take the marginal to be unimodal or Gaussian.  As one would expect, stronger assumptions on the eigenvalue decay result in polynomial learnability for broader classes of networks, but even mild eigenvalue decay will result in savings in runtime and sample complexity.  

The relationship between our assumption on eigenvalue decay and prior assumptions on the marginal distribution being Gaussian is similar in spirit to the dichotomy between the complexity of certain algorithmic problems on power-law graphs versus Erd\H{o}s-R\'enyi graphs.  Several important graph problems such as clique-finding become much easier when the underlying model is a random graph with appropriate power-law decay (as opposed to assuming the graph is generated from the classical $G(n,p)$ model) \cite{BrachCLS15,Krohmer2012}.  In this work we prove that neural network learning problems become tractable when the underlying distribution induces an empirical gram matrix with sufficiently strong eigenvalue-decay.

\subsection{Our Contributions}  Our main result is quite general and holds for any function class that can be suitably embedded in an RKHS (Reproducing Kernel Hilbert Space) with corresponding kernel function $k$ (we refer readers unfamiliar with kernel methods to \cite{scholkopf2001learning}).  Given $m$ draws from a distribution $(\textbf{x}_1, \ldots, \textbf{x}_m)$ and kernel $k$, recall that the {\em Gram matrix} $K$ is an $m \times m$ matrix where the $i,j$ entry equals $k(\textbf{x}_i,\textbf{x}_j)$.  For ease of presentation, we begin with an informal statement of our main result that highlights the relationship between the eigenvalue decay assumption and the run-time and sample complexity of our final algorithm. 
% We state our main theorem (informally) as follows:

\begin{theorem}[Informal] \label{thm:maini}
Fix function class ${\cal C}$ and kernel function $k$.  Assume ${\cal C}$ is approximated in the corresponding RKHS with norm bound $B$.  After drawing $m$ samples, let $K/m$ be the (normalized) $m \times m$ Gram matrix with eigenvalues $\{\lambda_{1},\ldots,\lambda_{m}\}$. For error parameter $\epsilon >0$,
\begin{enumerate}
\item If, for sufficiently large $i$, $\lambda_{i} \approx O(i^{-p})$, then ${\cal C}$ is efficiently learnable with $m = \tilde{O}(B^{1/p}/\epsilon^{2 + 3/p})$. 
\item If, for sufficiently large $i$, $\lambda_{i} \approx O(e^{-i})$, then ${\cal C}$ is efficiently learnable with $m = \tilde{O}(\log B/\epsilon^{2})$. 
\end{enumerate}
%In both cases above, the learning algorithm runs in polynomial time in all of the relevant parameters and $\log B$.  
\end{theorem} 

We allow a failure probability for the event that the eigenvalues do not decay.  In all prior work, the sample complexity $m$ depends linearly on $B$, and for many interesting concept classes (such as ReLUs), $B$ is exponential in one or more relevant parameters.  Given Theorem \ref{thm:maini}, we can use known structural results for embedding neural networks into an RKHS to estimate $B$ and take a corresponding eigenvalue decay assumption to obtain polynomial-time learnability.  Applying bounds recently obtained by Goel et al. \cite{goel2016reliably} we have

\begin{corollary}\label{cor:net}
Let ${\cal C}$ be the class of all fully-connected networks of ReLUs with one-hidden layer of $\ell$ hidden ReLU activations feeding into a single ReLU output activation (i.e., two hidden layers or depth-3).  Then, assuming eigenvalue decay of $O(i^{-\ell / \epsilon})$, ${\cal C}$ is learnable in polynomial time with respect to square loss on $\S^{n-1}$. If ReLU is replaced with sigmoid, then we require eigenvalue decay $O(i^{-\sqrt{\ell} \log(\sqrt{\ell}/\epsilon)})$.  
\end{corollary}

For higher depth networks, bounds on the required eigenvalue decay can be derived from structural results in \cite{goel2016reliably}.   Without taking an assumption, the fastest known algorithms for learning the above networks run in time exponential in the number of hidden units and accuracy parameter (but polynomial in the dimension) \cite{goel2016reliably}.

%Our results are sufficiently general to hold for a wide variety of domains and concept classes; for example we also obtain polynomial-time algorithms for learning well-studied Boolean concept classes such as DNF formulas under a sufficiently strong decay.

Our proof develops a novel approach for bounding the generalization error of kernel methods, namely we develop {\em compression schemes} tailor-made for classifiers induced by kernel-based regression, as opposed to current Rademacher-complexity based approaches.  Roughly, a compression scheme is a mapping from a training set $S$ to a small subsample $S'$ and {\em side-information} $\mathcal{I}$.  Given this compressed version of $S$, the decompression algorithm should be able to generate a classifier $h$.  In recent work, David, Moran and Yehudayoff \cite{david2016statistical} have observed that if the size of the compression is much less than $m$ (the number of samples), then the empirical error of $h$ on $S$ is close to its true error with high probability.

%\sgnote{Check this paragraph.} 

At the core of our compression scheme is a method for giving small description length (i.e., $o(m)$ bit complexity), approximate solutions to instances of kernel ridge regression.  Even though we assume $K$ has decaying eigenvalues, $K$ is neither sparse nor low-rank, and even a single column or row of $K$ has bit complexity at least $m$, since $K$ is an $m \times m$ matrix!  Nevertheless, we can prove that recent tools from Nystr\"{o}m sampling \cite{musco2016recursive} imply a type of sparsification for solutions of certain regression problems involving $K$.  Additionally, using preconditioning, we can bound the bit complexity of these solutions and obtain the desired compression scheme.  At each stage we must ensure that our compressed solutions do not lose too much accuracy, and this involves carefully analyzing various matrix approximations.  Our methods are the first compression-based generalization bounds for kernelized regression. 

%As indicated by Theorem \ref{thm:maini}, even mild eigenvalue decay will result in some savings in sample complexity.  
%The literature on matrix approximations and its applications in machine learning is now vast.  In particular, various methods for sparse or low-rank approximation have been applied to Gram %matrices, usually 
\subsection{Related Work} \label{sec:related} Kernel methods \cite{scholkopf2001learning} such as SVM, kernel ridge regression and kernel PCA have been extensively studied due to their excellent performance and strong theoretical properties.  For large data sets, however, many kernel methods become computationally expensive.  The literature on approximating the Gram matrix with the overarching goal of reducing the time and space complexity of kernel methods is now vast.  Various techniques such as random sampling \cite{williams2000using}, subspace embedding \cite{avron2014subspace}, and matrix factorization \cite{drineas2008relative} have been used to find a low-rank approximation that is efficient to compute and gives small approximation error.  The most relevant set of tools for our paper is Nystr\"{o}m sampling \cite{williams2000using,drineas2005nystrom}, which constructs an approximation of $K$ using a subset of the columns indicated by a selection matrix $S$ to generate a positive semi-definite approximation. Recent work on leverage scores have been used to improve the guarantees of Nystr\"{o}m sampling in order to obtain linear time algorithms for generating these approximations \cite{musco2016recursive}. 

The novelty of our approach is to use Nystr\"{o}m sampling in conjunction with compression schemes to give a new method for giving provable {\em generalization error} bounds for kernel methods. Compression schemes have typically been studied in the context of classification problems in PAC learning and for combinatorial problems related to VC dimension \cite{KuzminW07,LittWarmuth}.  Only recently have some authors considered compression schemes in a general, real-valued learning scenario \cite{david2016statistical}.  Cotter, Shalev-Shwartz, and Srebro have studied compression in the context of classification using SVMs and prove that for general distributions, compressing classifiers with low generalization error is not possible \cite{cotter2013learning}.  

The general phenomenon of eigenvalue decay of the Gram matrix has been studied from both a theoretical and applied perspective.  Some empirical studies of eigenvalue decay and related discussion can be found in \cite{MaB17,shawe2005eigenspectrum,TR14}.  There has also been prior work relating eigenvalue decay to generalization error in the context of SVMs or Kernel PCA (e.g., \cite{shawegeneralization,shawe2005eigenspectrum}).  Closely related notions to eigenvalue decay are that of {\em local Rademacher complexity} due to Bartlett, Bousquet, and Mendelson \cite{BBM05} (see also \cite{BM:2002}) and that of {\em effective dimensionality} due to Zhang \cite{zhang2002effective}.

The above works of Bartlett et al.~and Zhang give improved generalization bounds via data-dependent estimates of eigenvalue decay of the kernel.  At a high level, the goal of these works is to work under an assumption on the effective dimension and improve Rademacher-based generalization error bounds from $1/\sqrt{m}$ to $1/m$ ($m$ is the number of samples) for functions embedded in an RKHS of unit norm.  These works do not address the main obstacle of this paper, however, namely overcoming the complexity of the norm of the approximating RKHS. Their techniques are mostly incomparable even though the intent of using effective dimension as a measure of complexity is the same. 

Shamir has shown that for general linear prediction problems with respect to square-loss and norm bound $B$, a sample complexity of $\Omega(B)$ is required for gradient-based methods \cite{shamir2015sample}. Our work shows that eigenvalue decay can dramatically reduce this dependence, even in the context of kernel regression where we want to run in time polynomial in $n$, the dimension, rather than the (much larger) dimension of the RKHS.

%, for functions embedded in an RKHS with unit norm, 

% studied in the context of local Rademacher complexity due to Bartlett

%\sgnote{Added the Zhang comparison} Zhang \cite{zhang2002effective} gives generalization bounds for problems in Hilbert Space through %introducing the concept of effective dimensionality. Additionally, the elegant work on local Rademacher complexity due to Bartlett and% %Mendelson \cite{BM:2002} gives improved generalization via data-dependent estimates of eigenvalue decay of the Kernel. These works %improve the rate of convergence from $1/\sqrt{m}$ to $1/m$ based on an assumption on the effective dimension ($m$ is the sample %complexity), however, they do not address the main obstacle of this paper, 

\subsection{Recent work on Learning Neural Networks}
Due in part to the recent exciting developments in deep learning, there have been several works giving provable results for learning neural networks with various activations (threshold, sigmoid, or ReLU).  For the most part, these results take various assumptions on either 1) the distribution (e.g., Gaussian or Log-Concave) or 2) the structure of the network architecture (e.g. sparse, random, or non-overlapping weight vectors) or both and often have a bad dependence on one or more of the relevant parameters (dimension, number of hidden units, depth, or accuracy).  Another way to restrict the problem is to work only in the noiseless/realizable setting.  Works that fall into one or more of these categories include \cite{KlivansM13,ZhangLWJ15,XLS,JSA,SedghiA,ZhangPS17,Daniely17a}.  Kernel methods have been applied previously to learning neural networks \cite{zhang2016l1, LivniSS14, goel2016reliably,DFS16}.  The current broadest class of networks known to be learnable in fully polynomial-time in all parameters with no assumptions is due to Goel et al.~\cite{goel2016reliably}, who showed how to learn a sum of one hidden layer of sigmoids over the domain of $\S^{n-1}$, the unit sphere in $n$ dimensions.  We are not aware of other prior work that takes only a distributional assumption on the marginal and achieves fully polynomial-time algorithms for even simple networks (for example, one hidden layer of ReLUs).  

Much work has also focused on the ability of gradient descent to succeed in parameter estimation for learning neural networks under various assumptions with an intense focus on the structure of local versus global minima \cite{CHMAL15,Kawaguchi16,BG17,SoudryC16}.  Here we are interested in the traditional task of learning in the non-realizable or agnostic setting and allow ourselves to output a hypothesis outside the function class (i.e., we allow improper learning).  It is well known that for even simple neural networks, for example for learning a sigmoid with respect to square-loss, there may be many bad local minima \cite{AHW}.  Improper learning allows us to avoid these pitfalls. 

%in a non-realizable se

%I%f the Gram matrix showing decaying eigen-spectrum, these methods work best.

%T%he focus in the above mentioned works is largely to speed up the kernel methods, however, we take a different perspective by using these compression approaches to give better generalization errors for $B-$bounded norm regression in the kernel space when the kernel matrix has decaying eigen-spectrum. The parameter of concern in our work is the dependence of the sample complexity and subsequently run-time on $B$. We exploit the power of compression schemes for learning \cite{david2016statistical}.

%Another approach based on the decay of the eigen-spectrum is based on computing local rademacher complexity \cite{bartlett,mohri} which aims to improve the dependence of $m$ in the error term. In most works related to this, the parameter $B$ is not analyzed and generally assumed to be 1.

\section{Preliminaries}
\textbf{Notation.} The input space is denoted by $\mathcal{X}$ and the
output space is denoted by $\mathcal{Y}$. Vectors are represented with
boldface letters such as $\textbf{x}$. We denote a kernel function by
$k_{\psi}(x,x') = \langle \psi(x), \psi(x') \rangle$ where $\psi$ is
the associated feature map and for the kernel and $\mathcal{K}_\psi$ is the
corresponding reproducing kernel Hilbert space (RKHS).
% with inner product $\langle \cdot, \cdot \rangle$.
For necessary background material on kernel methods we refer the
reader to \cite{scholkopf2001learning}.

\subsection{Model and Generalization Bounds}
We will work in the general non-realizable model of statistical
learning theory also known as the {\em agnostic model of learning}.
In this model, the labels presented to the learner are arbitrary, and
the goal is to output a hypothesis that is competitive with the best
fitting function from some fixed class:

\begin{definition}[Agnostic Learning~\cite{KSS:1994, Haussler:1992}] 
	A concept class $\mathcal{C} \subseteq \mathcal{Y}^{\mathcal{X}}$ is agnostically learnable with respect to loss function $l : \mathcal{Y}^\prime \times \mathcal{Y} \rightarrow
	\mathbb{R}^+$ (where $\mathcal{Y} \subseteq \mathcal{Y}^\prime$) and distribution $D$ over $\mathcal{X} \times \mathcal{Y}$, if for every $\delta, \epsilon >
	0$ there exists a learning algorithm $\mathcal{A}$ given access
	to examples drawn from $D$, $\mathcal{A}$ outputs a hypothesis $h : \mathcal{X}
	\rightarrow \mathcal{Y}^\prime$, such that with probability at least $1 - \delta$,
	\begin{equation}
		\label{agnosticdefeq}  E_{(\textbf{x}, y) \sim D} [l(h(\textbf{x}), y)] \leq
		\min_{c \in C} E_{(\textbf{x}, y) \sim D} [l(c(\textbf{x}), y)] + \epsilon.
	\end{equation}
	Furthermore, we say that $\mathcal{C}$ is \emph{efficiently agnostically learnable to error
	$\epsilon$} if $\mathcal{A}$ can output an $h$ satisfying Equation
	\eqref{agnosticdefeq} with running time polynomial in $n$, $1/\epsilon$ and
	$1/\delta$.
\end{definition}

The agnostic model generalizes Valiant's PAC model of learning \cite{Val}, and so
all of our results will hold for PAC learning as well. 
The following is a well known theorem for proving generalization based on Rademacher complexity.
\begin{theorem}[\cite{BM:2002}] \label{generalizationbound}
	Let $\D$ be a distribution over $\mathcal{X} \times \mathcal{Y}$ and let $l : \mathcal{Y}^\prime
	\times \mathcal{Y}$ be a
	$b$-bounded loss function that is $L$-Lispschitz in its first argument.  Let
	$\mathcal{F}$ be a class of functions from $\mathcal{X}$ to $\mathcal{Y}^\prime$ and for any $f \in \mathcal{F}$, and $S = ((\textbf{x}_1, y_1), \ldots,  (\textbf{x}_m, y_m))  \sim
	\D^m$ and $\delta > 0$, with probability at least $1 - \delta$ we have,
	\[
		\left|E_{(x,y) \sim \D} [l(f(\textbf{x}), y)] - \frac{1}{m}\sum
_{i=1}^m l(f(\textbf{x}_i), y_i)\right| \leq 4 \cdot L \cdot \mathcal{R}_m(\mathcal{F})
		+ 2\cdot b \cdot \sqrt{\frac{\log (1/\delta)}{2m}}
	\]
	where $\mathcal{R}_m(\mathcal{F})$ is the Rademacher complexity of the function
	class $\mathcal{F}$. 
\end{theorem}
The Rademacher complexity of this linear class can be bounded by using the following theorem.
\begin{theorem}[\cite{KST:2008}] \label{rademachercomplexity}
	Let $\mathcal{K}$ be a subset of a Hilbert space equipped with inner product $\langle \cdot, \cdot \rangle$ such that for each $x \in \mathcal{K}$, $\langle \textbf{x}, \textbf{x}
	\rangle \leq X^2$, and let $\mathcal{W} = \{ \textbf{x} \rightarrow \langle \textbf{x} , \textbf{w} \rangle
	~|~ \langle \textbf{w}, \textbf{w} \rangle \leq W^2 \}$ be a class of linear functions.
	Then it holds that
		\[\mathcal{R}_m(\mathcal{W}) \leq X \cdot W \cdot \sqrt{\frac{1}{m}}.\]
\end{theorem}

\subsection{Selection and Compression Schemes}
It is well known that in the context of PAC learning Boolean function
classes, a suitable type of compression of the training data implies
learnability \cite{littlestone1986relating}.  Perhaps surprisingly,
the details regarding the relationship between compression and ceratin
other real-valued learning tasks have not been worked out until very
recently.  A convenient framework for us will be the notion of
compression and selection schemes due to David et al.~\cite{david2016statistical}.

A selection scheme is a pair of maps $(\kappa, \rho)$ where $\kappa$
is the selection map and $\rho$ is the reconstruction map. $\kappa$
takes as input a sample $\mathcal{S} = ((\textbf{x}_1, y_1), \ldots,
(\textbf{x}_m, y_m)) $ and outputs a sub-sample $\mathcal{S}'$ and a
finite binary string $b$ as side information. $\rho$ takes this input
and outputs a hypothesis $h$.  The {\em size} of the selection scheme is defined to be $k(m) = |\mathcal{S}'| + |b|$. We present a slightly modified version of the definition of an approximate compression scheme due to \cite{david2016statistical}:
\begin{definition}[$(\epsilon,\delta)$-approximate agnostic compression scheme]
A selection scheme $(\kappa, \rho)$ is an $(\epsilon,\delta)$-approximate agnostic compression scheme for hypothesis class $\mathcal{H}$ and sample satisfying property $P$ if for all samples $\mathcal{S}$ that satisfy $P$ with probability $1 - \delta$,  $f = \rho(\kappa(S))$ satisfies $\sum_{i=1}^m l(f(\textbf{x}_i), y_i) \leq \min_{h \in \mathcal{H}} \left(\sum_{i=1}^m l(h(\textbf{x}_i), y_i)\right) + \epsilon.$
\end{definition}

Compression has connections to learning in the general loss setting through the following theorem which shows that as long as $k(m)$ is small, the selection scheme generalizes.
\begin{theorem}[Theorem 30.2 \cite{shalev2014understanding}, Theorem 3.2 \cite{david2016statistical}] \label{thm_sel}
Let $(\kappa, \rho)$ be a selection scheme of size $k = k(m)$, and let $A_\mathcal{S} = \rho (\kappa (\mathcal{S}))$. Given $m$ i.i.d. samples drawn from any distribution $\D$ such that $k \leq m/2$, for constant bounded loss function $l : \mathcal{Y}^\prime \times \mathcal{Y} \rightarrow
	\mathbb{R}^+$ with probability $1 - \delta$, we have
\[\left| E_{(\textbf{x},y) \sim \D} [l(A_\mathcal{S}(x),y)] - \sum_{i=1}^m l(A_\mathcal{S}(\textbf{x}_i), y_i)\right| \leq \sqrt{\epsilon \cdot \left(\frac{1}{m}\sum_{i=1}^m l(A_\mathcal{S}(\textbf{x}_i), y_i)\right)} + \epsilon\]
where $\epsilon = 50 \cdot \frac{k \log (m/k) + \log(1/\delta)}{m}$.
\end{theorem}

\section{Problem Overview}
In this section we give a general outline for our main result. Let ${\cal S} = \{(\textbf{x}_1,y_1), \ldots, (\textbf{x}_m, y_m)\}$ be a training set of samples drawn i.i.d.~from some arbitrary distribution $\D$ on $\mathcal{X} \times [0,1]$ where $\mathcal{X} \subseteq \mathbb{R}^n$. Let us consider a concept class $\cal{C}$ such that for all $c \in \mathcal{C}$ and $\textbf{x} \in \mathcal{X}$ we have $c(\textbf{x}) \in [0,1]$. We wish to learn the concept class $\mathcal{C}$ with respect to the square loss, that is, we wish to find $c \in \mathcal{C}$ that approximately minimizes $E_{(\textbf{x},y) \sim \D} [(c(\textbf{x}) - y)^2]$. A common way of solving this is by solving the empirical minimization problem (ERM) given below and subsequently proving that it generalizes.
\begin{algorithm}[H]
	\caption{\label{alg:optprob1}}
\begin{align*}
		\underset{c \in \cal{C}}{\text{minimize}} \quad\quad &\frac{1}{m}\sum_{i=1}^m (c(\textbf{x}_i) - y_i)^2
\end{align*}
\end{algorithm}
Unfortunately, it may not be possible to efficiently solve the ERM in polynomial-time due to issues such as non-convexity. A way of tackling this is to show that the concept class can be approximately minimized by another hypothesis class of linear functions in a high dimensional feature space (this in turn presents new obstacles for proving generalization-error bounds, which is the focus of this paper).   
\begin{definition}[$\epsilon$-approximation]
\label{def_eps}
Let $\mathcal{C}_1$ and $\mathcal{C}_2$ be function classes mapping
domain $\mathcal{X}$ to $\mathbb{R}$.  $\mathcal{C}_1$ is $\epsilon$-approximated by $\mathcal{C}_2$
if for every $c \in \mathcal{C}_1$ there exists a $c' \in
\mathcal{C}_2$ such that for all $x \in \mathcal{X},  |c(x) - c'(x)| \leq \epsilon$.
%A class $\mathcal{C}_1$ is said to be $(\epsilon, m)$-minimized for loss $l$ over distribution $\D$ by class $\mathcal{C}_2$ if for all samples of size $m$, $\S = (x_i,y_i)_{i=1}^m$ drawn from $\D$ satisfies, 
%$ \min_{c_2 \in \mathcal{C}_2} \left(\frac{1}{m}\sum_{i=1}^m l(c_2(\textbf{x}_i),y_i)\right) \leq \min_{c_1 \in \mathcal{C}_1} \left(\frac{1}{m} \sum_{i=1}^ml(c_1(\textbf{x}_i),y_i)\right) + \epsilon. $
\end{definition}
%It is easy to see that if $\mathcal{C}_1$ is $\epsilon$-approximated with respect to $L_\infty$ by $\mathcal{C}_2$ then for any $m$, $\mathcal{C}_1$ is $(L\epsilon$, $m$)-minimized for by $\mathcal{C}_2$ for any $L$-lipschitz loss with respect to any distribution.

Suppose $\mathcal{C}$ can be $\epsilon$-approximated in the above sense by the hypothesis class $H_\psi = \{\textbf{x} \rightarrow \langle \textbf{v} , \psi(\textbf{x}) \rangle | \textbf{v} \in \mathcal{K}_\psi, \langle \textbf{v}, \textbf{v} \rangle \leq B\}$ for some $B$ and kernel function $k_{\psi}$. We further assume that the kernel is bounded, that is, $|k_\psi(\textbf{x},\textbf{\textbf{x'}})| \leq M$ for some $M > 0$ for all $\textbf{x},\textbf{x'} \in \mathcal{X}$. Thus, the problem relaxes to the following,
\begin{algorithm}[H]
	\caption{\label{opt}}
\begin{align*}
		\underset{v \in \mathcal{K}_{\psi}}{\text{minimize}}\quad \quad& \frac{1}{m}\sum_{i=1}^m (\langle \textbf{v}, \psi(\textbf{x}_i) \rangle - y_i)^2  &\text{subject to}\quad \quad& \langle \textbf{v}, \textbf{v} \rangle \leq B \nonumber
\end{align*}
\end{algorithm}
Using the Representer theorem, we have that the optimum solution for the above is of the form $\textbf{v}^* = \sum_{i=1}^m \alpha_i \psi(\textbf{x}_i)$ for some $\alpha \in \mathbb{R}^n$. Denoting the sample kernel matrix be $K$ such that $K_{i,j} = k_{\psi}(\textbf{x}_i, \textbf{x}_j)$, the above optimization problem is equivalent to the following optimization problem,
\begin{algorithm}[H]
	\caption{\label{opt_B}}
\begin{align*}
		\underset{\alpha \in \mathbb{R}^m}{\text{minimize}}\quad \quad& \frac{1}{m}||K \alpha - Y||_2^2  &\text{subject to}\quad \quad& \alpha^TK\alpha \leq B \nonumber
\end{align*}
\end{algorithm}
where $Y$ is the vector corresponding to all $y_i$ and $||Y||_{\infty} \leq 1$ since $\forall i \in [m], y_i \in [0,1]$. Let $\alpha_B$ be the optimal solution of the above problem. This is known to be efficiently solvable in $\mathsf{poly}(m,n)$ time as long as the kernel function is efficiently computable.

Applying Rademacher complexity bounds to $\mathcal{H}_\psi$ yields generalization error bounds that decrease, roughly, on the order of $B/\sqrt{m}$ (Theorem \ref{generalizationbound} and \ref{rademachercomplexity}).  If $B$ is exponential in $1/\epsilon$, the accuracy parameter, or in $n$, the dimension, as in the case of bounded depth networks of ReLUs, then this dependence leads to exponential sample complexity.  As mentioned in Section \ref{sec:related}, in the context of eigenvalue decay, various results \cite{zhang2002effective,BBM05,BM:2002} have been obtained to improve the dependence of $B/\sqrt{m}$ to $B/m$, but little is known about improving the dependence on $B$. %As mentioned earlier, Shamir~\cite{shamir2015sample} studied the dependence on the norm bound $B$ for linear predictors under square loss and showed that there exists a data distribution such that a linear dependence of $B$ on the sample complexity is necessary. 
%Thus, we need to make some assumptions on the distribution to improve this dependence.

Our goal is to show that eigenvalue decay of the empirical Gram matrix does yield generalization bounds with better dependence on $B$.  The key is to develop a novel compression scheme for kernelized ridge regression.  We give a step-by-step analysis for how to generate an approximate, compressed version of the solution to Optimization Problem 3.  Then, we will carefully analyze the bit complexity of our approximate solution and realize our compression scheme.  Finally, we can put everything together and show how quantitative bounds on eigenvalue decay directly translate into compressions schemes with low generalization error.  

% In the subsequent sections we show that we can approximate the matrix by a low rank matrix to generate a sparse solution which enables us to construct a compression scheme for the hypothesis class under consideration. To avoid large compression scheme size, we need to precondition the matrix. Finally, we show generalization results using our constructed compression scheme.
\section{Compressing the Kernel Solution}
Through a sequence of steps, we will sparsify $\alpha$ to find a
solution of much smaller bit complexity that is still an approximate
solution (to within a small additive error).  The quality and size of
the approximation will depend on the eigenvalue decay.

\subsection{Lagrangian Relaxation} We relax Optimization Problem \ref{opt_B} and consider the Lagrangian
version of the problem to account for the norm bound constraint.  This
version is convenient for us, as it has a nice closed-form solution.

%Under assumptions on the properties of eigenvalues of the sample kernel matrix, in this section we show how to compress the solution, that is, build a sparse solution for Optimization Problem \ref{opt_B}. We consider the Lagrangian version of the problem as follows.
\begin{algorithm}[H]
	\caption{\label{opt_l}}
\begin{align*}
		\underset{\alpha \in \mathbb{R}^m}{\text{minimize}}\quad \quad& \frac{1}{m}||K \alpha - Y||_2^2 + \lambda \alpha^TK\alpha 
\end{align*}
\end{algorithm}
We will later set $\lambda$ such that the error of considering this relaxation is small. It is easy to see that the optimal solution for the above lagrangian version is $\alpha = \left(K + \lambda m I\right)^{-1}Y$.

\subsection{Preconditioning}
To avoid extremely small or non-zero eigenvalues, we consider a
perturbed version of $K$, $K_\gamma = K + \gamma m I$. This gives us
that the eigenvalues of $K_\gamma$ are always greater than or equal to
$\gamma m$. This property is useful for us in our later analysis.
Henceforth, we consider the following optimization problem on the
perturbed version of K:
\begin{algorithm}[H]
	\caption{\label{opt_l_pert}}
\begin{align*}
		\underset{\alpha \in \mathbb{R}^m}{\text{minimize}}\quad \quad& \frac{1}{m}||K_\gamma \alpha - Y||_2^2 + \lambda \alpha^TK_\gamma\alpha 
\end{align*}
\end{algorithm}
The optimal solution for perturbed version is $\alpha_{\gamma} = \left(K_{\gamma} + \lambda  m I\right)^{-1}Y = \left(K + (\lambda 
+ \gamma) m I\right)^{-1}Y$.

\subsection{Sparsifying the Solution via Nystr\"om Sampling}
We will now use tools from Nystr\"om Sampling to sparsify the solution
obtained from Optimzation Problem \ref{opt_l_pert}.  To do so, we
first recall the definition of effective
dimension or degrees of freedom for the kernel
\cite{zhang2002effective}:

\begin{definition}[$\eta$-effective dimension]
For a positive semidefinite $m \times m$ matrix $K$ and parameter $\eta$, the $\eta$-effective dimension of $K$ is defined as $d_{\eta}(K) = tr(K (K + \eta m I)^{-1})$.
\end{definition}
Various kernel approximation results have relied on this quantity, and
here we state a recent result due to \cite{musco2016recursive} who gave the first application independent result that shows that there is an efficient way of computing a set of columns of $K$ such that $\bar{K}$, a matrix constructed from the columns is close in terms of 2-norm to the matrix $K$. More formally,
\begin{theorem}[\cite{musco2016recursive}]
\label{thm_approx}
For kernel matrix $K$, there exists an algorithm that gives a set of
\\ $O\left(d_{\eta}(K) \log \left(d_{\eta}(K)/\delta\right) \right)$ columns, such that $\bar{K} = KS(S^TKS)^{\dagger}S^TK$ where $S$ is the matrix that selects the specific columns, satisfies with probability $1 - \delta$, $\bar{K} \preceq K \preceq \bar{K} + \eta m I$.
\end{theorem}
It can be shown that $\bar{K}$ is positive semi-definite. Also, the above implies $||K - \bar{K}||_2 \leq \eta m$.
We use the decay to approximate the Kernel matrix with a low-rank matrix constructed using the columns of $K$. Let $\bar{K}_{\gamma}$ be the matrix obtained by applying Theorem \ref{thm_approx} to $K_{\gamma}$ for $\eta> 0$ and consider the following optimization problem,
\begin{algorithm}[H]
	\caption{\label{opt_l_s}}
\begin{align*}
		\underset{\alpha \in \mathbb{R}^m}{\text{minimize}}\quad \quad& \frac{1}{m}||\bar{K}_{\gamma} \alpha - Y||_2^2 + \lambda \alpha^T\bar{K}_{\gamma}\alpha 
\end{align*}
\end{algorithm}
The optimal solution for the above is $\bar{\alpha}_{\gamma} =
\left(\bar{K}_{\gamma} + \lambda m I\right)^{-1}Y$. Since
$\bar{K}_{\gamma} =
K_{\gamma}S(S^TK_{\gamma}S)^{\dagger}S^TK_{\gamma}$, solving for the
above enables us to get a solution $\alpha^* =
S(S^TK_{\gamma}S)^{\dagger}S^TK_{\gamma} \bar{\alpha}_{\gamma}$, which
is a $k$-sparse vector for $k = O\left(d_{\eta}(K_{\gamma}) \log \left(d_{\eta}(K_{\gamma})/\delta\right) \right)$.

\subsection{Bounding the Error of the Sparse Solution}
We bound the additional error incurred by our sparse hypothesis $\alpha^*$  compared to $\alpha_B$. To do so, we bound the error for each of the approximations: sparsification, preconditioning and lagrangian relaxation in the following lemma.
\begin{lemma} \label{lem_errors}
The errors due to the following approximations can be bounded as follows.
\begin{enumerate}
\item Error due to sparsification: $||\bar{K}_{\gamma}\bar{\alpha}_{\gamma} - Y||_2 \leq ||K_{\gamma}\alpha_{\gamma} - Y||_2 + \frac{\eta \sqrt{m}}{\lambda + \gamma}$
\item Error due to preconditioning: $||K_{\gamma}\alpha_{\gamma} - Y||_2 \leq ||K\alpha - Y||_2 + \frac{\gamma  \sqrt{m}}{\lambda + \gamma}$
\item Error due to lagrangian relaxation: $||K\alpha - Y||_2 \leq ||K\alpha_B - Y||_2 + \sqrt{\lambda m B}$
\end{enumerate}
\end{lemma}
\begin{proof}
The errors can be bounded as follows.
\begin{enumerate}
\item We have,
\begin{align}
||\bar{K}_{\gamma}\bar{\alpha}_{\gamma} &- Y||_2 - ||K_{\gamma}\alpha_{\gamma} - Y||_2 \nonumber\\
& \leq ||\bar{K}_{\gamma}\bar{\alpha}_{\gamma} - K_{\gamma} \alpha_{\gamma}||_2 \label{eq_tr} \\
& = ||\bar{K}_{\gamma}\left(\bar{K}_{\gamma} + \lambda m I\right)^{-1}Y - K_{\gamma}\left(K_{\gamma} + \lambda m I\right)^{-1}Y||_2 \label{eq_subs}\\
& = \lambda m||\left(-\left(\bar{K}_{\gamma} + \lambda m I\right)^{-1} + \left(K_\gamma + \lambda m I\right)^{-1}\right)Y||_2 \label{eq_prop}\\
& = \lambda m ||\left(\bar{K}_{\gamma} + \lambda m I\right)^{-1}\left(\bar{K}_{\gamma} - K_{\gamma}\right)\left(K_\gamma + \lambda m I\right)^{-1}Y||_2 \label{eq_inv}\\
& \leq \lambda m ||\left(\bar{K}_{\gamma} + \lambda m I\right)^{-1}||_2||\bar{K}_{\gamma} - K_{\gamma}||_2||\left(K + (\lambda + \gamma) m I\right)^{-1}||_2||Y||_2 \label{eq_norm}\\
& \leq \frac{||\bar{K}_{\gamma} - K_{\gamma}||_2}{(\lambda + \gamma) \sqrt{m}} \leq \frac{\eta \sqrt{m}}{\lambda + \gamma} \label{eq_bound}.
\end{align}
Here \ref{eq_tr} follows from triangle inequality, \ref{eq_subs} follows from substitution and \ref{eq_prop} follows from using $A\left(A + cI\right)^{-1} = \left(A + cI - cI\right)\left(A + cI\right)^{-1} = I - c\left(A + cI\right)^{-1}$. \ref{eq_inv} follows from $a^{-1} - b^{-1} = - a^{-1}\left(a - b\right) b^ {-1}$ and \ref{eq_norm} follows from $||AB||_2 \leq ||A||_2||B||_2$. Lastly \ref{eq_bound} follows from $||A^{-1}||_2 = \lambda_{min}\left(A\right)^{-1}$, $\lambda_{min}\left(A + c I\right) \geq c$ for psd $A$. We also use $K_\gamma = K + \gamma m I$ and  $||Y||_2 \leq \sqrt{m}$.
\item Similar to the above proof, we have,
\begin{align}
||K_{\gamma}\alpha_{\gamma} &- Y||_2 - ||K\alpha - Y||_2 \nonumber \\
& \leq ||K_{\gamma} \alpha_{\gamma} - K(K + \lambda m I)^{-1}Y||_2 \\
& = ||K_{\gamma}\left(K_{\gamma} + \lambda m I\right)^{-1}Y - K\left(K + \lambda m I\right)^{-1}Y||_2 \\
& = \lambda m ||\left(K_{\gamma} + \lambda m I\right)^{-1}\left(K_{\gamma} - K\right)\left(K + \lambda m I\right)^{-1}Y||_2 \\
& \leq \lambda m ||\left(K + (\lambda + \gamma) m I\right)^{-1}||_2||\gamma m I||_2||\left(K + \lambda m I\right)^{-1}||_2||Y||_2 \\
& \leq \frac{\gamma  \sqrt{m}}{\lambda + \gamma} \label{eq_pert}.
\end{align}
\item Since $\alpha$ minimizes Optimization Problem 4, we have 
\begin{align}
||K\alpha - Y||_2^2 \leq &||K\alpha - Y||_2^2 + \lambda m \alpha^T K \alpha \\
&\leq ||K\alpha_B - Y||_2^2 + \lambda m \alpha_B^T K \alpha_B \\
&\leq ||K\alpha_B - Y||_2^2 + \lambda m B
\end{align}
where the last inequality follows from $\alpha_B^T K \alpha_B \leq B$ by the constraint of the bounded optimization problem. Taking the square-root, we get,
\begin{align}
||K\alpha - Y||_2 \leq \sqrt{||K\alpha_B - Y||_2^2 + \lambda m B} \leq ||K\alpha_B - Y||_2 + \sqrt{\lambda m B} \label{eq_B}
\end{align}
\end{enumerate}
\end{proof}
We now combine the above to give the following theorem.
\begin{theorem}[Total Error]
\label{thm_tot}For $\lambda = \frac{\epsilon^2}{81B}$, $\eta \leq \frac{\epsilon^3}{729B}$ and $\gamma \leq \frac{\epsilon^3}{729B}$, we have
\[\frac{1}{m}||K_\gamma \alpha^* - Y||_2^2 \leq \frac{1}{m}||K\alpha_B - Y||_2^2 + \epsilon.\]
\end{theorem}
\begin{proof}
Note that $\bar{K}\bar{\alpha}_{\gamma} = K_\gamma \alpha^*$ by the definition of $\alpha^*$, from the previous lemma, we have,
\begin{align}
||\bar{K}\bar{\alpha}_{\gamma} - Y||_2 - ||K\alpha_B - Y||_2 \leq \frac{\eta \sqrt{m}}{\lambda + \gamma} + \frac{\gamma\sqrt{m}} {\lambda + \gamma} + \sqrt{\lambda m B} = \beta
\end{align}
where $\beta = \frac{(\eta + \gamma)\sqrt{m}} {\lambda + \gamma} + \sqrt{\lambda m B}$. Squaring and then dividing by $m$ on both sides, we get
\begin{align}
\frac{1}{m}||\bar{K}_\gamma\bar{\alpha}_{\gamma} - Y||_2^2 & \leq \frac{1}{m}||K\alpha_B - Y||_2^2 + 2 \frac{\beta}{m} ||K\alpha_B - Y||_2 + \frac{\beta^2}{m}\\
& \leq \frac{1}{m}||K\alpha_B - Y||_2^2 + 2 \frac{\beta}{\sqrt{m}} + \frac{\beta^2}{m}\\
& \leq \frac{1}{m}||K\alpha_B - Y||_2^2 + 3 \frac{\beta}{\sqrt{m}}
\end{align}
The second inequality follows from $||K\alpha_B - Y||_2^2 \leq ||Y||_2^2 \leq m$ since 0 is a feasible solution for Optimization Problem 3. The last inequality follows from assuming $\frac{\beta}{\sqrt{m}} \leq 1$ which holds for our choice of $\beta$. Setting the values in the lemma satisfies the last inequality gives us $\beta \leq \frac{\epsilon\sqrt{m}}{3}$ giving us the desired bound.
\end{proof}

\subsection{Computing the Sparsity of the Solution}
To compute the sparsity of the solution, we need to bound $d_{\eta}(K_{\beta})$. We consider the following different eigenvalue decays.
\begin{definition}[Eigenvalue Decay]\label{def_eigen}
Let the real eigenvalues of a symmetric $m \times m$ matrix $A$ be denoted by $\lambda_1 \geq \cdots \geq \lambda_m$.
\begin{enumerate}
\item $A$ is said to have \textbf{$(C,p)$-polynomial eigenvalue decay} if for all $i \in \{1, \ldots, m\}$, $\lambda_i \leq C i^{-p}$.
\item $A$ is said to have \textbf{$C$-exponential eigenvalue decay} if for all $i \in \{1, \ldots, m\}$, $\lambda_i \leq Ce^{-i}$.
\end{enumerate}
\end{definition}

Note that in the above definitions $C$ and $p$ are not necessarily constants. We allow $C$ and $p$ to depend on other parameters (the choice of these parameters will be made explicit in subsequent theorem statements). We can now bound the effective dimension in terms of eigenvalue decay:

\begin{theorem}[Bounding effective dimension] \label{lem_eff} For
  $\gamma m \leq \eta $, the $\eta$-effective dimension of $K_\gamma$
  can be bounded as follows,
\begin{enumerate}
\item If $K/m$ has \textbf{$(C, p)$-polynomial eigenvalue decay} for $p > 1$ then $d_{\eta}(K_{\gamma}) \leq \left(\frac{C}{(p-1)\eta} \right)^{1/p} + 2$.
\item If $K/m$ has \textbf{$C$-exponential eigenvalue decay} then $d_{\eta}(K_{\gamma}) \leq \log \left(\frac{C}{(e -1)\eta}\right) + 2$.
\end{enumerate}
\end{theorem}
\begin{proof}
Observe that,
\begin{align*}
d_{\eta}(K_{\gamma}) &= tr(K_{\gamma}(K_{\gamma} + \eta m I)^{-1})\\
& = \sum_{i=1}^m \frac{\lambda_i(K_{\gamma})}{\lambda_i(K_{\gamma}) + \eta m} \\
& \leq \sum_{i=1}^j \frac{\lambda_i(K_{\gamma})}{\lambda_i(K_{\gamma})} + \sum_{i=j+1}^m \frac{\lambda_i(K_{\gamma})}{\eta m} \\
& \leq j + \sum_{i=j+1}^m \frac{\gamma m + \lambda_i(K)}{\eta m } \\
& \leq j + 1 + \sum_{i=j+1}^m \frac{\lambda_i(K)}{\eta m } 
\end{align*}
Here the second equality follows from trace of matrix being equal to the sum of the eigenvalues and the last follows from $\gamma m \leq \eta$.
\begin{enumerate}
\item For $(C, p)$-polynomial eigenvalue decay with $p > 1$,
\[\sum_{i=k+1}^m \frac{\lambda_i(K)}{\eta m } = \sum_{i=k+1}^m \frac{C i^{-p}}{\eta} \leq \frac{C}{\eta}\int_{k+1}^{\infty} i^{-p} di = \frac{C(k+1)^{-p + 1}}{(p -1)\eta} \]
Substituting $j = \left(\frac{C}{(p-1)\eta} \right)^{1/p}$ we get the required bound.
\item For $C$-exponential eigenvalue decay,
\[\sum_{i=k+1}^m \frac{\lambda_i(K)}{\eta m } = \sum_{i=k+1}^m \frac{C e^{-i}}{\eta} \leq \sum_{i=k+1}^\infty \frac{C e^{-i}}{\eta} = \frac{Ce^{-k}}{(e - 1)\eta} \]
Substituting $j = \log \left(\frac{C}{(e -1)\eta}\right)$ we get the required bound.
\end{enumerate}
\end{proof}
\noindent\textbf{Remark}: \textit{Based on the above analysis, observe that we only need the eigenvalue decay to hold after the $j$th eigenvalue for $j$ defined above. Thus the top $j-1$ eigenvalues need not be constrained.}
\section{Bounding the Size of the Compression Scheme}
The above analysis gives us a sparse solution for the problem and, in
turn, an $\epsilon$-approximation for the error on the overall sample
$\cal{S}$ with probability $1 - \delta$.  We can now fully define our
compression scheme for the hypothesis class $H_\psi$ with respect to
samples satisfying the eigenvalue decay property.
\begin{itemize}
\item \textbf{Selection Scheme $\kappa$}: Given input $\mathcal{S} = (\textbf{x}_i,y_i)_{i=1}^m$,
\begin{enumerate}
\item Use RLS-Nystr\"{o}m Sampling \cite{musco2016recursive} to compute $\bar{K_\gamma} = K_\gamma S(S^TK_\gamma S)^{\dagger}S^TK_\gamma$ for $\eta = \frac{\epsilon^3}{5832 B}$ and $\gamma = \frac{\epsilon^3}{5832 Bm}$. Let $\mathcal{I}$ be the sub-sample corresponding to the columns selected using $S$.
\item Solve Optimization Problem \ref{opt_l_s} for $\lambda = \frac{\epsilon^2}{324B}$ to get $\bar{\alpha}_{\gamma}$.
\item Compute the $|\mathcal{I}|$-sparse vector $\alpha^* = S(S^TK_\gamma S)^{\dagger}S^TK_\gamma \bar{\alpha}_{\gamma} = K_\gamma^{-1}\bar{K}_\gamma \bar{\alpha}_{\gamma}$ ($K_\gamma$ is invertible as all eigenvalues are non-zero).
\item Output subsample $\mathcal{I}$ along with $\tilde{\alpha}^*$ which is $\alpha^*$ truncated to precision $\frac{\epsilon}{4 M |\mathcal{I}|}$ per non-zero index.
\end{enumerate}
\item \textbf{Reconstruction Scheme $\rho$}: Given input subsample $\mathcal{I}$ and $\tilde{\alpha}^*$, output hypothesis,
\[h_{\mathcal{S}}(\textbf{x}) = clip_{0,1}(\textbf{w}^T\tilde{\alpha}^*)\]
where $\textbf{w}$ is a vector with entries $K(\textbf{x}_i,\textbf{x}) + \gamma m \mathbbm{1}[\textbf{x} = \textbf{x}_i]$ for $i \in \mathcal{I}$ and 0 otherwise where $\gamma = \frac{\epsilon^3}{5832Bm}$. Note, $clip_{a,b}(x) = \max(a, \min(b,x))$ for some $a < b$.
\end{itemize}
The size of the above scheme can be bounded using the following lemma.
\begin{lemma}
The bit complexity of the side information of the selection scheme $\kappa$ given above is $O\left(d \log \left(\frac{d}{\delta}\right)\log\left(\frac{\sqrt{m}BM d \log(d/\delta)}{\epsilon^4}\right)\right)$ where $d$ is the $\eta$-effective dimension of $K_{\gamma}$ for $\eta = \frac{\epsilon^3}{5832B}$ and $\gamma = \frac{\epsilon^3}{5832Bm}$.
\end{lemma}
\begin{proof}

From the selection scheme we can bound the norm of $\alpha^* = K_\gamma^{-1}\bar{K_\gamma}\bar{\alpha}_{\gamma}$ for $\gamma =\frac{ \epsilon^3}{5832Bm}$, the side information, as follows,
\begin{align}
||\alpha^*||_2 &= ||K_\gamma^{-1}\bar{K_\gamma}\bar{\alpha}_{\gamma}||_2\\
&= ||K_\gamma^{-1}\bar{K_\gamma}(\bar{K_\gamma} + \lambda mI)^{-1}Y||_2 \\
&\leq ||K_\gamma^{-1}||_2||\bar{K_\gamma}(\bar{K_\gamma} + \lambda mI)^{-1}||_2||Y||_2 \\
&\leq \frac{1}{\gamma m}\cdot 1 \cdot \sqrt{m}\\
& = \frac{1}{\gamma \sqrt{m}} = \frac{5832\sqrt{m}B}{\epsilon^3}.
\end{align}

Thus we can upper bound the bit complexity of the non-decimal part of $\alpha^*$ as,
\begin{align*}
\sum_{i\in \mathcal{I}}\log\left(|\alpha^*_i|\right) &= \frac{1}{2}\sum_{i=1}^{|\mathcal{I}|}\log\left(\left(\alpha^*_i\right)^2\right) \\ &\leq \frac{|\mathcal{I}|}{2}\log\left(\frac{\sum_{i=1}^{|\mathcal{I}|} \left(\alpha^*_i\right)^2}{|\mathcal{I}|}\right)\\ 
&\leq |\mathcal{I}|\log \left(\frac{||\alpha^*||_2}{\sqrt{|\mathcal{I}|}}\right) \leq |\mathcal{I}|\log\left(\frac{5832\sqrt{m}B}{\epsilon^3}\right)
\end{align*}
where $|\mathcal{I}| = O\left(d \log \left(\frac{d}{\delta}\right)\right)$ according to Theorem \ref{thm_approx}. Since each non-zero index has $\frac{\epsilon}{4 M |\mathcal{I}|}$ precision, we need $|\mathcal{I}| \log \left(\frac{4 M |\mathcal{I}|}{\epsilon} \right)$ bits for the decimal part. Combining the two-parts we get the required bound.
\end{proof}

The following theorem shows that the above is a compression scheme for $\mathcal{H}_\psi$.
\begin{theorem}
\label{thm_compress}
$(\kappa, \rho)$ is an $(\epsilon, \delta)$-approximate agnostic compression scheme for the hypothesis class $H_\psi$ for sample $\mathcal{S}$ of size $k(m, \epsilon, \delta, B, M) =\ O\left(d \log \left(\frac{d}{\delta}\right)\log\left(\frac{\sqrt{m}BM d \log(d/\delta)}{\epsilon^4}\right)\right)$ where $d$ is the $\eta$-effective dimension of $K_{\gamma}$ for $\eta = \frac{\epsilon^3}{5832B}$ and $\gamma = \frac{\epsilon^3}{5832Bm}$.
\end{theorem}
\begin{proof}
For $\mathcal{S} = (\textbf{x}_i,y_i)_{i=1}^m$ and $h_{\cal S}$ the output of the compression scheme, we have
\begin{align}
\frac{1}{m} \sum_{i=1}^m (h_{\cal S}(\textbf{x}_i) - y_i)^2 &\leq \frac{1}{m} \sum_{i=1}^m \left( \sum_{j \in \mathcal{I}} (K(\textbf{x}_j,\textbf{x}_i) + \gamma m \mathbbm{1}[\textbf{x}_j = \textbf{x}_i]) \tilde{\alpha}^*_j - y_i\right)^2 \label{eq_clip} \\
&\leq \frac{1}{m} \sum_{i=1}^m \left( \sum_{j \in \mathcal{I}} (K(\textbf{x}_j,\textbf{x}_i) + \gamma m\mathbbm{1}[\textbf{x}_j = \textbf{x}_i]) \alpha^*_j - y_i\right)^2 + \frac{\epsilon}{2} \label{eq_alpha_approx}\\
& = \frac{1}{m} ||K_\gamma\alpha^* - Y||_2^2  + \frac{\epsilon}{2} \label{eq_matrix}\\
& = \frac{1}{m} ||\bar{K}_\gamma \bar{\alpha}_\gamma - Y||_2^2  + \frac{\epsilon}{2} \label{eq_alpha_star} \\
& = \frac{1}{m} ||K \alpha_B - Y||_2^2  + \frac{\epsilon}{2} + \frac{\epsilon}{2} \label{eq_tot_err} \\
& = \min_{h \in H_\psi} \left( \frac{1}{m} \sum_{i=1}^m (h(\textbf{x}_i) - y_i)^2 \right) + \epsilon \label{eq_best}
\end{align}
Here \ref{eq_clip} follows from the fact that since the output is in $[0,1]$ clipping only reduces the loss, \ref{eq_alpha_approx} follows from the precision used while compressing and since square loss is 2-Lipschitz, \ref{eq_matrix} follows from representing it in the matrix form, \ref{eq_alpha_star} follows since $\alpha^* = K_\gamma^{-1}\bar{K}_\gamma \bar{\alpha}_\gamma$ by definition, \ref{eq_tot_err} follows from Theorem \ref{thm_tot} with the given parameters satisfying the theorem for $\epsilon/2$ and lastly \ref{eq_best} follows from the definition of $\alpha_B$. Thus, this gives us our result.
\end{proof}

\section{Putting It All Together: From Compression to Learning}
We now present our final algorithm: \textit{Compressed Kernel
  Regression} (Algorithm \ref{alg_1}). 
\setcounter{algorithm}{0}
\begin{algorithm}
\floatname{algorithm}{Algorithm}
  \caption{Compressed Kernel Regression\label{alg_1}}
  \begin{algorithmic}[1]
    \Statex \textbf{Input}: Samples $\mathcal{S} = (\textbf{x}_i,y_i)_{i=1}^m$, gram matrix $K$ on $\mathcal{S}$, constants $\epsilon, \delta > 0$, norm bound $B$ and maximum kernel function value $M$ on $\mathcal{X}$.
     \State Using RLS-Nystr\"{o}m Sampling \cite{musco2016recursive} with input $(K_\gamma, \eta m)$ for $\gamma = \frac{\epsilon^3}{5832Bm}$ and $\eta = \frac{\epsilon^3}{5832B}$ compute $\bar{K_\gamma} = K_\gamma S(S^TK_\gamma S)^{\dagger}S^TK_\gamma$. Let $\mathcal{I}$ be the subsample corresponding to the columns selected using $S$. Note that the number of columns selected depends on the $\eta$ effective dimension of $K_\gamma$.
	\State Solve Optimization Problem \ref{opt_l_s} for $\lambda = \frac{\epsilon^2}{324B}$ to get $\bar{\alpha}_{\gamma}$ over $\mathcal{S}$
      \State  Compute $\alpha^* = S(S^TK_\gamma S)^{\dagger}S^TK_\gamma \bar{\alpha}_{\gamma} = K_\gamma^{-1}\bar{K}_\gamma \bar{\alpha}_{\gamma}$
\State Compute $\tilde{\alpha}^*$ by truncating each entry of $\alpha^*$ up to precision $\frac{\epsilon}{4 M |\mathcal{I}|}$
\Statex \textbf{Output}: $h_{\mathcal{S}}$ such that for all $x \in \mathcal{X}$, $h_{\mathcal{S}}(\textbf{x}) = clip_{0,1}(\textbf{w}^T\tilde{\alpha}^*)$ where $\textbf{w}$ is a vector with entries $K(\textbf{x}_i,\textbf{x}) + \gamma m \mathbbm{1}[\textbf{x} = \textbf{x}_i]$ for $i \in \mathcal{I}$ and 0 otherwise.
\end{algorithmic}
\end{algorithm}
Note that the algorithm is efficient and takes at most $O(m^3)$ time.

For our learnability result, we restrict distributions to those that satisfy eigenvalue decay. More formally,
\begin{definition}[Distribution Satisfying Eigenvalue Decay]
\label{def_eigen_d}
Consider distribution $\D$ over $\mathcal{X}$ and kernel function $k_{\psi}$. Let ${\cal S}$ be a sample drawn i.i.d. from the distribution $\D$ and $K$ be the empirical gram matrix corresponding to kernel function $k_{\psi}$ on $\mathcal{S}$.
\begin{itemize}
\item $\D$ is said to satisfy $(C,p,N)$-polynomial eigenvalue decay if with probability $1-\delta$ over the drawn sample of size $m \geq N$ , $K/m$ satisfies $(C,p)$-polynomial eigenvalue decay.
\item $\D$ is said to satisfy $(C, N)$-exponential eigenvalue decay if with probability $1-\delta$ over the drawn sample of size $m \geq N$, $K/m$ satisfies $C$-exponential eigenvalue decay.
\end{itemize}
\end{definition}
Our main theorem proves generalization of the hypothesis output by Algorithm \ref{alg_1} for distributions satisfying eigenvalue decay in the above sense.
\begin{theorem}[Formal for Theorem \ref{thm:maini}] \label{thm:B}
Fix function class ${\cal C}$ with output bounded in $[0,1]$ and $M$-bounded kernel function $k_{\psi}$ such that ${\cal C}$ is $\epsilon_0$-approximated by $H_\psi = \{\textbf{x} \rightarrow \langle \textbf{v} , \psi(\textbf{x}) \rangle | \textbf{v} \in \mathcal{K}_\psi, \langle \textbf{v}, \textbf{v} \rangle \leq B\}$ for some $\psi, B$. Consider a sample ${\cal S} = \{(\textbf{x}_i,y_i)_{i=1}^{m}\}$ drawn i.i.d. from $\D$ on $\mathcal{X} \times [0,1]$. There exists an algorithm $\mathcal{A}$ that outputs hypothesis $h_{\cal S} = \mathcal{A}(\mathcal{S})$, such that,
\begin{enumerate}
\item If $\D_\mathcal{X}$ satisfies $(C,p,m)$-polynomial eigenvalue decay with probability $1- \delta/4$ then with probability $1-\delta$ for $m = \tilde{O}((CB)^{1/p}\log(M)/\epsilon^{2 + 3/p})$,
\[\mathbb{E}_{(\textbf{x},y) \sim \D} (h_{\cal S}(\textbf{x}) - y)^2 \leq \min_{c\in \mathcal{C}}\left( \mathbb{E}_{(\textbf{x},y) \sim \D}(c(\textbf{x}) - y)^2 \right) + 2\epsilon_0 + \epsilon
\]
\item If $\D_\mathcal{X}$ satisfies $(C,m)$-exponential eigenvalue decay with probability $1- \delta/4$ then with probability $1- \delta$ for $m = \tilde{O}(\log CB\log(M)/\epsilon^{2})$,
\[\mathbb{E}_{(\textbf{x},y) \sim \D} (h_{\cal S}(\textbf{x}) - y)^2 \leq \min_{c\in \mathcal{C}}\left( \mathbb{E}_{(\textbf{x},y) \sim \D}(c(\textbf{x}) - y)^2 \right) + 2\epsilon_0 + \epsilon
\]
\end{enumerate}
Algorithm $\mathcal{A}$ runs in time $\mathsf{poly}(m,n)$.
\end{theorem} 
\begin{proof}
Since $\mathcal{C}$ is $\epsilon_0$-approximated by $H_\psi$ we have,
\[\min_{h \in H_\psi}\left( \frac{1}{m}\sum_{i=1}^m (h(\textbf{x}_i) - y_i)^2 \right) \leq \min_{c \in \mathcal{C}}\left( \frac{1}{m}\sum_{i=1}^m (c(\textbf{x}_i) - y_i)^2 \right) + 2\epsilon_0 \leq \frac{1}{m}\sum_{i=1}^m (c^*(\textbf{x}_i) - y_i)^2 + 2\epsilon_0\]
where $c^* \in \mathcal{C}$ be such that it minimizes $\mathbb{E}_{(x,y) \sim \D} (c(x) - y)^2$ over all $c \in \mathcal{C}$. The first inequality follows from square loss being 2-Lipschitz and the last inequality follows from $c^*$ being a feasible solution.

Let $K$ be the empirical gram matrix corresponding to $k_\psi$ on $\mathcal{S}$. Let $h_{\cal S}$ be the hypothesis output by Algorithm \ref{alg_1} with input $(\mathcal{S}, K, \epsilon_1, \delta/4, B, M)$ for $\epsilon_1>0$ chosen later. From Theorem \ref{thm_compress}  with probability $1 - \delta/4$, we have
\[\frac{1}{m}\sum_{i=1}^m (h_{\cal S}(\textbf{x}_i) - y_i)^2 \leq \min_{h \in H_\psi}\left(\frac{1}{m}\sum_{i=1}^m (h(\textbf{x}_i) - y_i)^2 \right) + \epsilon_1.\]

We know that for every $c \in \mathcal{C}$, the square loss is bounded by 1, thus using Chernoff-Hoeffding inequality, with probability $1 - \delta/4$, we have 
\begin{align*}
\frac{1}{m}\sum_{i=1}^m (c^*(\textbf{x}_i) - y_i)^2 &\leq  \mathbb{E}_{(\textbf{x},y) \sim \D} (c^*(\textbf{x}) - y)^2 + \epsilon_2
\end{align*}
where $\epsilon_2 = \sqrt{\frac{\log (4/\delta)}{2m}}$.

Now the output of $h_{\cal S}$ lies in $[0,1]$ thus for all $(\textbf{x},y)$, $(y - h_{\cal S}(\textbf{x}))^2$ lies in $[0,1]$. Thus viewing $h_{\cal S}$ as the output of the compression scheme $(\kappa, \rho)$ of size $k$ (Theorem \ref{thm_compress}), by Theorem \ref{thm_sel}, we have with probability $1 - \delta/4$,
\[\left|\mathbb{E}_{(\textbf{x},y) \sim \D} (h_{\cal S}(\textbf{x}) - y)^2 - \frac{1}{m}\sum_{i=1}^m (h_S(\textbf{x}_i) - y_i)^2\right| \leq \sqrt{\frac{\epsilon_3}{m}\sum_{i=1}^m (h_{\cal S}(\textbf{x}_i) - y_i)^2} + \epsilon_3 \leq \epsilon_3 + \sqrt{\epsilon_3} \leq 2\sqrt{\epsilon_3}\]
where $\epsilon_3 = 50 \cdot\frac{k \log (m/k) + \log (4/\delta)}{m}$.

Combining the above, we have with probability $1 - \delta$,
\begin{align}
\mathbb{E}_{(\textbf{x},y) \sim \D} (h_{\cal S}(\textbf{x}) - y)^2 &\leq \frac{1}{m}\sum_{i=1}^m (h_S(\textbf{x}_i) - y_i)^2 + 2\sqrt{\epsilon_3} \\
& \leq \min_{h \in H_\psi}\left( \frac{1}{m}\sum_{i=1}^m (h(\textbf{x}_i) - y_i)^2 \right) + \epsilon_1 + 2\sqrt{\epsilon_3} \\
& \leq \frac{1}{m}\sum_{i=1}^m (c^*(\textbf{x}_i) - y_i)^2 + 2\epsilon_0 + \epsilon_1 + 2\sqrt{\epsilon_3} \\
& \leq \min_{c \in \mathcal{C}}\left( \mathbb{E}_{(\textbf{x},y) \sim \D} (c(\textbf{x}) - y)^2 \right) + 2\epsilon_0 + \epsilon_1 + \epsilon_2 + 2\sqrt{\epsilon_3}
\end{align}
Using Theorem \ref{lem_eff} we can bound $k$ depending on the different eigenvalue decay assumption. Now we set $\epsilon_1 = \epsilon/3$ and substituting for $m$. Recall that $\epsilon_2$ and $\epsilon_3$ are functions of $m$ and for the chosen $m$, they are bounded by $\epsilon/3$ giving us the desired bound. Since Algorithm \ref{alg_1} runs in time $\mathsf{poly(m,n)}$ we get the required time complexity.
\end{proof}
\noindent {\bf Remark:} The above theorem can be extended to different rates of eigenvalue decay.  For example, it can be shown that {\em finite} rank $r$ would give a bound independent of $B$ but dependent instead on $r$.  Also, as in the proof of Theorem \ref{lem_eff}, it suffices for the eigenvalue decay to hold only for $i$ sufficiently large. 

\section{Learning Neural Networks}
Here we apply our main theorem to the problem of learning neural networks. For technical definitions of neural networks, we refer the reader to \cite{zhang2016l1}. We define the class of neural networks as follows.
\begin{definition}[Class of Neural Networks \cite{goel2016reliably}]
Let $\mathcal{N}[\sigma, D, W, T]$ be the class of fully-connected, feed-forward networks with $D$ hidden layers, activation function $\sigma$ and quantities $W$ and $T$ described as follows:
\begin{enumerate}
\item Weight vectors in layer $0$ have $2$-norm bounded by $T$.
\item Weight vectors in layers $1, \ldots, D$ have $1$-norm bounded by $W$.
\item For each hidden unit $\sigma(\textbf{w} \cdot \textbf{z})$ in the network, we have $|\textbf{w}\cdot \textbf{z}| \leq T$ (by $\textbf{z}$ we denote the input feeding into unit $\sigma$ from the previous layer).
\end{enumerate}
\end{definition}
We consider activation functions $\sigma_{relu}(x) = \max(0, x)$ and $\sigma_{sig} = \frac{1}{1 + e^{-x}}$, though other activation functions fit within our framework.  Goel et al.~\cite{goel2016reliably} showed that the class of ReLUs/Sigmoids along with their compositions can be approximated by linear functions in a high dimensional Hilbert space (corresponding to a particular type of polynomial kernel). We use the following theorem that follows directly from the structural results in \cite{goel2016reliably} (and uses the composed-kernel technique of Zhang et al.~\cite{zhang2016l1}).
\begin{theorem} \label{thm:approximation}
Consider the following hypothesis class $\mathcal{H}_{\mathsf{MK}_d} = \{\textbf{x} \rightarrow \langle \textbf{v} , \psi(\textbf{x}) \rangle | \textbf{v} \in \mathcal{K}_{\mathsf{MK}_d}, \langle \textbf{v}, \textbf{v} \rangle \leq B\}$ where $\mathcal{K}_{\mathsf{MK}_d}$ is the Hilbert space corresponding to the Multinomial Kernel \footnote{The multinomial kernel defined by \cite{goel2016reliably} is $\mathsf{MK}_{d}(\textbf{x},\textbf{x}') = \sum_{i=0}^d (\textbf{x} \cdot \textbf{x}')^i$.} and $\psi$ is the corresponding feature vector. For $D>0$, consider the composed class $\mathcal{H}^{(D)} = \{\textbf{x} \rightarrow \langle v , \psi^{(D)}(\textbf{x}) \rangle | \textbf{v} \in \mathcal{K}^{(D)}, \langle \textbf{v}, \textbf{v} \rangle \leq B\}$ where $\psi^{(D)}$ is the feature vector of the $D$-times composed kernel $K^{(D)}$ \footnote{\cite{zhang2016l1} defined kernel $K^{(1)}(\textbf{x}, \textbf{x}\prime) = \frac{1}{2 - (\textbf{x} \cdot\textbf{x}\prime)}$. The corresponding composed kernel function is defined as $K^{(D)}(\textbf{x},\textbf{x}\prime) = \frac{1}{2 - K^{(D-1)}(\textbf{x},\textbf{x}\prime)}$}. Then for $\mathcal{X} = \S^{n-1}$,
\begin{enumerate}
\item \textbf{Single ReLU}: $\mathcal{C}_{relu} = \mathcal{N}[\sigma_{relu}, 0, \cdot, 1]$ is $\epsilon$-approximated by $\mathcal{H}_d$ for $d = O(1/\epsilon)$ and $B = 2^{(\tau/\epsilon)}$ with $M= d + 1$,
\item \textbf{Network of ReLUs}: $\mathcal{C}_{relu-D} = \mathcal{N}[\sigma_{relu}, D, W, T]$ is $\epsilon$-approximated by $\mathcal{H}_{(D)}$ for $B = 2^{(\tau W^DDT/\epsilon)^{D}}$ with $M = 2$,
\item \textbf{Network of Sigmoids}: $\mathcal{C}_{sig-D} = \mathcal{N}[\sigma_{sig}, D, W, T]$ is $\epsilon$-approximated by $\mathcal{H}_{(D)}$ for $B = 2^{(\tau T\log(W^DD/\epsilon))^{D}}$ with $M=2$,
\end{enumerate}
for some sufficiently large constant $\tau >0$.
\end{theorem}
As mentioned earlier, the sample complexity of prior work depends linearly on $B$, which, for even a single ReLU, is exponential in $1/\epsilon$.  Assuming sufficiently strong eigenvalue decay, we can show that we can obtain fully polynomial time algorithms for the above classes.
\begin{theorem} \label{thm:net}
For $\epsilon, \delta > 0$,  consider $\D$ on $\S^{n-1} \times [0,1]$ such that,
\begin{enumerate}
\item For $\mathcal{C}_{relu}$, $\D_\mathcal{X}$ satisfies $(C,p,m)$-polynomial eigenvalue decay for $p  \geq \xi/\epsilon$,
\item For $\mathcal{C}_{relu-D}$, $\D_\mathcal{X}$ satisfies $(C,p,m)$-polynomial eigenvalue decay for $p \geq (\xi W^DDT/\epsilon)^{D}$,
\item For $\mathcal{C}_{sig-D}$,  $\D_\mathcal{X}$ satisfies $(C,p,m)$-polynomial eigenvalue decay for $p \geq (\xi T\log(W^DD/\epsilon)))^{D}$,
\end{enumerate}
where $\D_\mathcal{X}$ is the marginal distribution on $\mathcal{X} = \S^{n-1}$, $\xi > 0$ is some sufficiently large constant and $C \leq (n \cdot 1/\epsilon)^{\zeta p}$ for some constant $\zeta > 0$. The value of $m$ is obtained from Theorem \ref{thm:B} as $m = \tilde{O}((CB)^{1/p}\log(M)/\epsilon^{2 + 3/p})$ where the values of $B,M$ are derived from Theorem \ref{thm:approximation}.

%, \log(1/\delta))$ where the degree inside the $\mathsf{poly}$ is $O(p)$.

Each decay assumption above implies an algorithm for agnostically learning the corresponding class on $\S^{n-1}\times [0,1]$ with respect to the square loss in time $\mathsf{poly}(n, 1/\epsilon, \log(1/\delta))$.  \end{theorem} 

\begin{proof}  The proof follows from applying Theorem \ref{thm:B} to the appropriate kernel from Theorem \ref{thm:approximation} and substituting the corresponding eigenvalue decays to compute the sample size needed by Algorithm \ref{alg_1} for learnability. For example, for the case of single ReLU, $M = \mathsf{poly}(1/\epsilon)$, $B = 2^{(\tau/\epsilon)}$ and we take $p \geq \xi/\epsilon$.  So for any $C =  (n \cdot 1/\epsilon)^{\zeta p}$, we obtain sample complexity $m = \tilde{O}((C 2^{(\tau/\epsilon)})^{1/p}\log(M)/\epsilon^{2 + 3/p}) = \mathsf{poly}(n, 1/\epsilon)$. Since the algorithm takes time at most $\mathsf{poly}(m,n)$, we obtain the required result. 
\end{proof}
Note that assuming an exponential eigenvalue decay (stronger than polynomial) will result in efficient learnability for much broader classes of networks. 

Since it is not known how to agnostically learn even a single ReLU with respect to arbitrary distributions on $\S^{n-1}$ in polynomial-time\footnote{Goel et al. \cite{goel2016reliably} show that agnostically learning a single ReLU over $\{-1,1\}^{n}$ is as hard as learning sparse parities with noise.  This reduction can be extended to the case of distributions over $\S^{n-1}$ \cite{simonschat}.}, much less a network of ReLUs, we state the following corollary highlighting the decay we require to obtain efficient learnability for simple networks:

\begin{corollary}[Restating Corollary \ref{cor:net}]
Let ${\cal C}$ be the class of all fully-connected networks of ReLUs with one-hidden layer of size $\ell$ feeding into a final output ReLU activation where the $2$-norms of all weight vectors are bounded by $1$.  Then, (suppressing the parameter $m$ for simplicity), assuming $(C,i^{-\ell / \epsilon})$-polynomial eigenvalue decay for $C = \mathsf{poly}(n,1/\epsilon,\ell)$, ${\cal C}$ is learnable in polynomial time with respect to square loss on $\S^{n-1}$. If ReLU is replaced with sigmoid, then we require eigenvalue decay of $i^{-\sqrt{\ell} \log(\sqrt{\ell}/\epsilon)}$.  
\end{corollary}
\begin{proof}
By assumption the $2$-norm of each weight vector is bounded by $1$, which implies that the $1$-norm of the weight vector to the one hidden unit at layer two is at most $\sqrt{\ell}$. Also observe that, the maximum $2$-norm of any input vector $\textbf{z}$ to a hidden unit with weight vector $\textbf{w}$ is bounded by $\sqrt{\ell}$ hence $|\textbf{w} \cdot \textbf{x}| \leq \sqrt{\ell}$. Using these properties we can apply Theorem \ref{thm:net} with parameters $W= \sqrt{\ell}$, $T=\sqrt{\ell}$ and $D=1$ to obtain the required result.
\end{proof}

\section{Conclusions and Future Work}
We have proposed the first set of distributional assumptions that guarantee fully polynomial-time algorithms for learning expressive classes of neural networks (without restricting the structure of the network).  The key abstraction was that of a {\em compression scheme} for kernel approximations, specifically Nystr\"om sampling. We proved that eigenvalue decay of the Gram matrix reduces the dependence on the norm $B$ in the kernel regression problem.  

Prior distributional assumptions, such as the underlying marginal equaling a Gaussian, neither lead to fully polynomial-time algorithms nor are representative of real-world data sets\footnote{Despite these limitations, we still think uniform or Gaussian assumptions are worthwhile and have provided highly nontrivial learning results.}.  Eigenvalue decay, on the other hand, has been observed in practice and does lead to provably efficient algorithms for learning neural networks.  

A natural criticism of our assumption is that the rate of eigenvalue decay we require is too strong.  In some cases, especially for large depth networks with many hidden units, this may be true\footnote{It is useful to keep in mind that agnostically learning even a single ReLU with respect to all distributions seems computationally intractable, and that our required eigenvalue decay in this case is only a function of the accuracy parameter $\epsilon$.}.  Note, however, that our results show that even moderate eigenvalue decay will lead to improved algorithms.  Further, it is quite possible our assumptions can be relaxed.  An obvious question for future work is what is the minimal rate of eigenvalue decay needed for efficient learnability? Another direction would be to understand how these eigenvalue decay assumptions relate to other distributional assumptions.

%Finally, 

\medskip \noindent \textbf{Acknowledgements.}
We would like to thank Misha Belkin and Nikhil Srivastava for very helpful conversations regarding kernel ridge regression and eigenvalue decay.  We also thank Daniel Hsu, Karthik Sridharan, and Justin Thaler for useful feedback.  The analogy between eigenvalue decay and power-law graphs is due to Raghu Meka.
\bibliography{paper}
\bibliographystyle{plain}

\end{document}